\documentclass[letterpaper, 10 pt, conference]{ieeeconf}  

\IEEEoverridecommandlockouts                              

\usepackage{amsfonts, dsfont, mathtools, amsmath, amssymb, bbold, mathrsfs, bm, upgreek}
\usepackage{graphicx, float, epsfig, color, psfrag, adjustbox, enumerate}
\usepackage[font=small]{caption}
\usepackage{refcount, url, cleveref}
\usepackage[noadjust]{cite}
\usepackage[usenames,dvipsnames,svgnames,table]{xcolor}

\usepackage{soul}

\title{\LARGE \bf
Kinematics of continuum planar grasping 
}

\author{Udit Halder, Nicolas Echeverria Zambrano, Xincheng Li
\thanks{All authors are with the Department of Mechanical and Aerospace Engineering, University of South Florida. 
  Corresponding e-mail:  {\tt\small udithalder@usf.edu}}%
}
\def\R{{\mathds{R}}}

\def\0{{\mathbb{0}}}
\def\1{{\mathds{1}}}

\newcommand{\norm}[1]{\left\lVert#1\right\rVert}

\definecolor{db}{RGB}{23,20,119}
\definecolor{dg}{RGB}{2,101,15}

\newtheorem{proposition}{Proposition}[section]
\newtheorem{definition}{Definition}
\newtheorem{remark}{Remark}

\usepackage{upgreek}
\newcommand{\dif}{\mathrm{d}}
\newcommand{\transpose}{\intercal}
\newcommand{\set}[1]{\left\{#1\right\}}

\newcommand{\material}[1]{
	\ifthenelse{\equal{#1}{\kappa}}{\upkappa}{
	\ifthenelse{\equal{#1}{\nu}}{\upnu}{
	\ifthenelse{\equal{#1}{\omega}}{\upomega}{
	\ifthenelse{\equal{#1}{\sigma}}{\upsigma}{
	\ifthenelse{\equal{#1}{\theta}}{\uptheta}{
	\mathsf{#1}}}}}}
}

\newcommand{\object}{\text{o}}

\newcommand{\Tangent}{\bm{\mathsf{t}}}
\newcommand{\Normal}{\bm{\mathsf{n}}}

\newcommand{\unitx}{\bm{\mathsf{e}}_1}
\newcommand{\unity}{\bm{\mathsf{e}}_2}

\graphicspath{{figures/}}

\begin{document}
\bstctlcite{BSTcontrol} 
\maketitle
\thispagestyle{empty}
\pagestyle{empty}


\begin{abstract}
This paper presents an analytical framework to study the geometry arising when a soft continuum arm grasps a planar object. 
Both the arm centerline and the object boundary are modeled as smooth curves. The grasping problem is formulated as a kinematic boundary following problem, in which the object boundary acts as the arm’s `shadow curve'. This formulation leads to a set of reduced kinematic equations expressed in terms of relative geometric shape variables, with the arm curvature serving as the control input. 
An optimal control problem is formulated to determine feasible arm shapes that achieve optimal grasping configurations, and its solution is obtained using  Pontryagin's Maximum Principle.
Based on the resulting optimal grasp kinematics, a class of continuum grasp quality metrics is proposed using the algebraic properties of the associated continuum grasp map. Feedback control aspects in the dynamic setting are also discussed.
The proposed methodology is illustrated through systematic numerical simulations.

\end{abstract}


\section{Introduction} \label{sec:intro}

Continuum manipulators, inspired by the flexible and dexterous appendages of organisms such as octopus arms, elephant trunks, or plant tendrils, offer unique advantages for manipulation in cluttered, uncertain, or dynamically changing environments~\cite{rus2015design, tekinalp2024topology}. Unlike traditional rigid-link manipulators, continuum arms feature redundant degrees of freedom and inherent compliance, enabling them to wrap around objects and establish stable grasps through distributed contact~\cite{li2013autonomous, mehrkish2021comprehensive, wang2024spirobs}. 
This is in contrast to traditional robotic grasping, which typically relies on a finite set of discrete contact points generated by rigid fingers~\cite{murray1994mathematical, bicchi2000robotic_short}.

Modeling and control of soft robotic manipulation have therefore become active areas of research in the fields of robotics and control theory. Theoretical modeling approaches of soft manipulators range from piecewise constant curvature models~\cite{webster2010design} to more sophisticated Cosserat rod formalisms~\cite{antman1995nonlinear, chang2023energy}. Several model-based control  strategies have also been explored~\cite{della2023model, wang2021optimal, wang2022sensory, chang2023energy, wang2025neuralbiocyber}. Although progress has been made in designing soft grasping manipulators~\cite{li2013autonomous, wang2024spirobs, kim2022physics}, explicit modeling and control design for grasping by a soft arm remained scarce in the literature~\cite{haibin2018modeling, xun2024cosserat,  halder2025statics}.

This paper builds upon our prior work~\cite{halder2025statics}, where the static conditions of a planar object being grasped were studied. The distributed contact forces along the boundary of the object were considered and the equilibrium of the object was expressed by using a `continuum grasp map' that transforms the local contact forces to the object frame of reference. Even though this approach gave rise to useful system-theoretic analogs, the kinematic and dynamic properties of the  soft arm wrapping around the object were ignored. 

The key focus of this work is to study the kinematics (or the relative geometry) of a soft arm grasping a planar object. Another major goal is to effectively design the shape of the arm, based on the object boundary. The specific contributions are as follows: 

\smallskip
\noindent
{\bf 1) Modeling the grasp kinematics.} 
We model the grasping problem as a boundary following problem, where we view the boundary as a `shadow curve' of the arm centerline (like a Bertrand mate~\cite{zhang2004boundary}). Then, we reduce the kinematic equations of the arm and the boundary in the space of relative geometric shape variables (e.g., distance and bearing to the boundary), where the arm curvature serves as control input. This leads to a (curvature) control problem for synthesizing desirable grasping configurations.

\smallskip
\noindent
{\bf 2) Designing optimal grasps and assessing their qualities.} We pose the grasp synthesis problem as an optimization problem to generate the optimal arm shape. 
We obtain the first order necessary optimality conditions using Pontryagin's Maximum Principle~\cite{liberzon2011calculus} and numerically compute the solutions for three different boundary shapes. Furthermore, we define grasp quality metrics which depend on algebraic properties of the continuum grasp map and systematically compare these metrics for each object boundary. These metrics may be viewed as continuum analogs of their discrete counterparts~\cite{roa2015grasp, rubert2018characterisation}.

\smallskip
\noindent
{\bf 3) Proposing a dynamic feedback control.} Finally, we explore feedback to solve the grasp synthesis problem and provide a pathway toward a dynamic control design. Modeling the arm as a Cosserat rod~\cite{antman1995nonlinear, chang2023energy}, we propose a feedback control law to dynamically stabilize the arm to a grasping configuration and provide guarantees of convergence (Prop.~\ref{prop:dyanmic_feedback}). We also produce numerical results using a high-fidelity simulation platform \textit{PyElastica}~\cite{gazzola2018forward, PyElastica}.

\smallskip
The remainder of this paper is structured as follows.
Section~\ref{sec:model} presents the modeling of the grasp kinematics, Sec.~\ref{sec:control_optimal} discusses optimal grasping configurations, which leads to the introduction of continuum grasp quality metrics in Sec.~\ref{sec:quality}. A direction toward utilizing this framework in the dynamic case is provided in Sec.~\ref{sec:dynamics} and the paper is concluded in Sec.~\ref{sec:conclusion}.  

\begin{figure*}[t]
	\centering
	\includegraphics[width=1\textwidth]{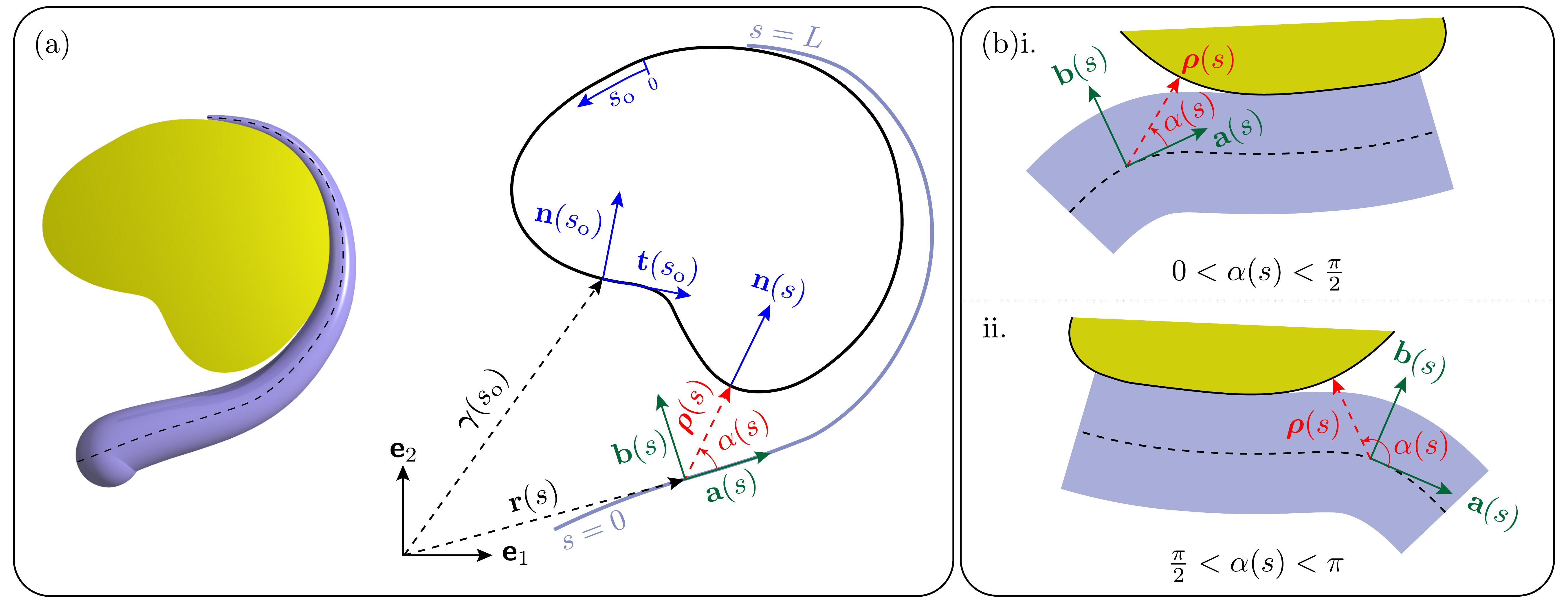}
	\caption{(a) Schematic of a continuum arm grasping a planar object with smooth boundary. The boundary of the object $\gamma$ is parameterized by its arclength $s_\object$. The arm is described by the position of its centerline $\mathbf{r}(s) \in \R^2$ where $s\in [0, L]$ is the arclength of the centerline. The proposed control is given by the contact vector $\bm{\rho}(s)$ and the contact angle $\alpha(s)$ defined in \eqref{eq:rho_alpha_definition}. (b) An illustration showing the arm: i. approaching the target when $0< \alpha < \pi/2$; ii. departing the target when $\pi/2 < \alpha < \pi$.}  
	\label{fig:modeling}
    \vspace*{-10pt}
\end{figure*}

\section{Contact kinematics} \label{sec:model}
In this section, we present kinematic models for a continuum arm and an object in a plane spanned by a fixed laboratory frame $\{\unitx, \unity \}$ (see a representative illustration in Fig.~\ref{fig:modeling}(a)). 

\subsection{Continuum arm kinematics}
A soft continuum arm is modeled as a planar Cosserat rod \cite{antman1995nonlinear, chang2020energy}. 
Moreover, the inextensibility and unshearability conditions are assumed for simplicity (Kirchhoff rod). 

The geometry and kinematics of the continuum arm is entirely described by its centerline curve, parameterized by its arclength $s$. 
The position vector of the centerline is denoted by $\mathbf{r}(s) \in \R^2, s\in [0, L]$, where $L$ is the length of the arm. The angle $\theta(s) \in S^1$ describes the material frame spanned by the orthonormal basis $\set{\mathbf{a}, \mathbf{b}}$, where $\mathbf{a} = \cos \theta \,\mathbf{e}_1 + \sin \theta \, \mathbf{e}_2, ~ \mathbf{b} = -\sin \theta \, \mathbf{e}_1 + \cos \theta \, \mathbf{e}_2$. The kinematics of the arm are then given by
\begin{equation}
	\frac{\dif \mathbf{r}}{\dif s} = \begin{bmatrix} \cos \theta \\ \sin \theta \end{bmatrix} = \mathbf{a}, \quad \frac{\dif \theta}{\dif s} = \kappa
	\label{eq:kinematics}
\end{equation}
where $\kappa(s)$ is called the curvature of the arm centerline.

\subsection{Defining the object}

We consider the object to be a \textit{fixed} planar rigid body. We assume that it has a smooth boundary, i.e. its boundary is represented by a planar smooth curve, parameterized by its arclength $s_\object$. Let the position of the boundary be denoted by $\bm{\gamma}(s_\object) \in \R^2, s_\object \in [0, L_\object]$, where $L_\object$ is the length of the boundary (we take $L_\object \geq L$). The angle $\phi(s_\object) \in S^1$ describes the moving tangent-normal frame $\{\Tangent(s_\object), \Normal(s_\object)\}$ as $\Tangent(s_\object) = \cos \phi \, \unitx + \sin \phi \, \unity$ and $\Normal(s_\object) = -\sin \phi \, \unitx + \cos \phi \, \unity$. Then, the curve $\bm{\gamma}$ is described by
\begin{align}
\begin{split}
\frac{\dif \bm{\gamma} }{\dif s_\object}  &= \begin{bmatrix} \cos \phi  \\ \sin \phi  \end{bmatrix} = \Tangent , \quad 
\frac{\dif \phi}{\dif s_\object} = \kappa_\object 
\end{split}
\label{eq:gamma}
\end{align}
where $\kappa_\object (s_\object)$ is the curvature of the object. Finally, notice that $\bm{\gamma}$ is a simple closed curve, so that $\bm{\gamma} (0) = \bm{\gamma} (L_\object)$.

\subsection{Continuum contact kinematics}

In this section, we study the kinematics of the soft arm being in continuum contact with the planar object. Moreover, we only consider the case where the object is toward the left of the arm, since the other case (where the object is toward the right) can be dealt with analogously. The continuum contact is formally defined as follows. 

For each arm arclength $s \in [0, L]$, a point of interest on the object boundary is the closest to $\mathbf{r}(s)$. Denote the closest boundary point as $\bm{\gamma}(s)$. Notice that due to (potentially) non-zero curvatures of the arm and the object, the point $\bm{\gamma}(s)$ does not necessarily move along the boundary with a unit speed as in \eqref{eq:gamma}. In other words, the two arclengths  of the arm and the object are related by a scaling factor as follows
\begin{align}
\frac{\dif s_\object}{\dif s} = \nu_\object
\end{align}
Then, the derivatives in the equations \eqref{eq:gamma} can be recast with respect to $s$ by multiplying the right hand sides by the scaling factor $\nu_\object (s)$. 

\begin{remark} 
{\bf The `shadow curve'.}
The two curves $\bm{\gamma}(s)$ and $\mathbf{r}(s)$ can be thought of as trajectories of two point particles moving in the plane with different speeds, the particle on the object at speed $\nu_\object$ and the particle on the arm at unit speed $\nu_{\text{arm}} = 1$ (see \eqref{eq:kinematics}). Thus, the curve $\bm{\gamma}(s),\, s \in [0, L]$, is viewed as an accompanying `shadow curve' to the arm\footnote{This is also called a Bertrand mate in the literature~\cite{Bertrand1850, zhang2004boundary}}.    
\end{remark}

Let the contact vector $\bm{\rho}(s)$, contact distance $\rho (s)$, and the contact angle $\alpha (s)$ (see Fig.~\ref{fig:modeling}(a)) be defined as follows
\begin{align}
\begin{split}
	\bm{\rho}(s) := \bm{\gamma}(s) - \mathbf{r}(s), &\quad \rho(s) := \norm{\bm{\rho}(s)} \\
	\mathsf{R}(\alpha(s)) \, {\mathbf{a}(s)} &:= \frac{\bm{\rho}(s)}{\rho(s)}
\end{split}
\label{eq:rho_alpha_definition}
\end{align}
where $\mathsf{R}(\alpha)$ is the planar rotation matrix for rotating a vector in $\R^2$ counterclockwise by the angle $\alpha$. We also define the contact depth $\delta(s)$ as
\begin{align}
\delta (s) := r(s) - \rho(s)
\label{eq:delta}
\end{align}
where $r(s)$ is the radius of arm cross-section (assumed to be circular). We then have the following definition.
\smallskip
\begin{definition} \label{def:contact}
{\bf (Continuum contact)} The arm is said to be in contact with the object at arclength $s$ if $\delta (s) \geq 0$.
\end{definition}

\smallskip
\begin{remark} \label{remark:contact_forces}
{\bf Contact forces.} At the contact point $\bm{\gamma}(s)$, a normal contact force in the direction perpendicular to the surface tangent (i.e., along the $\Normal (s)$ vector) is generated~\cite{murray1994mathematical, halder2025statics}. A frictional force along the direction of the tangent $\Tangent (s)$ is also generated due to the friction between the two bodies in contact. Modeling the contact forces is the subject of the field of contact mechanics~\cite{johnson1987contact, popov2019handbook}. Typically, the contact depth $\delta$ is one of the leading factors that determines the normal contact force. Detailed discussion of this topic is postponed for a future work. 
\end{remark}



The evolution of the variables $\rho$ and $\alpha$ are obtained from their definitions~\eqref{eq:rho_alpha_definition} as~\cite{halder2016steering, wang2022sensory}
\begin{align}
\begin{split}
\frac{\dif \rho}{\dif s} &= - \cos \alpha \\
\frac{\dif \alpha}{\dif s} &= - \kappa + \frac{1}{\rho} (-\nu_\object + \sin \alpha)
\end{split}
\label{eq:rho_alpha_1}
\end{align}
It remains to calculate the value of $\nu_\object$, for which we proceed as follows. 
Since $\bm{\gamma}(s)$ is by definition the closest boundary point to $\mathbf{r}(s)$, we necessarily have $\bm{\rho}(s) \cdot \Tangent (s) = 0$,
where `$\cdot$' denotes the dot product. Differentiating both sides 
of this condition with respect to $s$ and using $\mathbf{a}\cdot \Tangent = \cos (\pi/2 - \alpha) = \sin \alpha$, we obtain 
\begin{align}
\nu_\object = \frac{\sin \alpha}{1+\rho \kappa_\object} 
\label{eq:object_speed}
\end{align} 
Finally, plugging this result in \eqref{eq:rho_alpha_1}, we obtain
\begin{align}
\begin{split}
\frac{\dif \rho}{\dif s} &= - \cos \alpha \\
\frac{\dif \alpha}{\dif s} &= - \kappa + \frac{\kappa_\object}{1+\rho \kappa_\object} \sin \alpha
\end{split}
\label{eq:rho_alpha_2}
\end{align}
Equations~\eqref{eq:rho_alpha_2} are regarded as the contact kinematic equations, where the arm curvature $\kappa(s)$ acts as a control variable.

For realistic contact situations, we assume that $\alpha(s) \in (0, \pi),~\forall s$. In Fig.~\ref{fig:modeling}(b), we illustrate two possible contact scenarios $0< \alpha \leq \tfrac{\pi}{2}$ and $\tfrac{\pi}{2} < \alpha < \pi$. On the other hand, we assume a positive $\rho$, i.e. $\delta (s) < r(s), ~ \forall s$. Therefore, only such arm curvatures $\kappa (s)$ are allowed so that the resulting $\rho$ and $\alpha$ profiles satisfy these conditions.

Furthermore, we acknowledge a singularity in the expression of $\nu_\object$~\eqref{eq:object_speed}, when $1+\rho \kappa_\object = 0$. This may happen if the object curvature is negative enough. To avoid such situations, we assume that the object curvature $\kappa_\object$ is `nice enough' that the shadow object particle always \textit{moves forward} along the boundary, i.e. $\nu_\object (s) > 0, \forall s$ (see~\cite{zhang2004boundary} for a detailed discussion on this topic). Notice that under the constraints on $(\rho, \alpha)$, $\nu_\object (s) > 0$ implies $1 + \rho (s) \kappa_\object (s) > 0$. In particular, this is the second order optimality condition for the point $\bm{\gamma}(s)$ being the closest to $\mathbf{r}(s)$.

\subsection{Curvature control problem for continuum contact} \label{sec:control_problem}
The kinematic equations of continuum contact \eqref{eq:rho_alpha_2} can be viewed as a non-autonomous (because of the $\kappa_\object (s)$ term) control system, where the arm curvature $\kappa(s)$ is considered as a control input. The control goal then is to choose the arm curvature function $\kappa(\cdot)$ for the kinematics~\eqref{eq:rho_alpha_2} to track a given profile $(\rho_d (\cdot), \alpha_d (\cdot))$. 

The underlying motivation for the control problem is two-fold:

(1) Soft continuum arms are usually actuated by internal musculature, that alters the arm shape (curvature) through the elastodynamics of the arm~\cite{chang2021controlling, tekinalp2024topology}. Since we are only interested in the kinematics here, it is therefore reasonable to directly treat the curvature as the control.

(2) For a given grasping task, it becomes necessary to generate a specific continuum profile of contact forces (see e.g.~\cite{halder2025statics}). As pointed out in Remark~\ref{remark:contact_forces}, the generated contact forces depend on the contact depth $\delta$ (therefore on $\rho$) and on contact angle $\alpha$. Thus, it is reasonable to track a desired $(\rho_d, \alpha_d)$ profile.


\section{Optimal grasping configuration} \label{sec:control_optimal}


\begin{figure*}[t]
	\centering
	\includegraphics[width=1\textwidth]{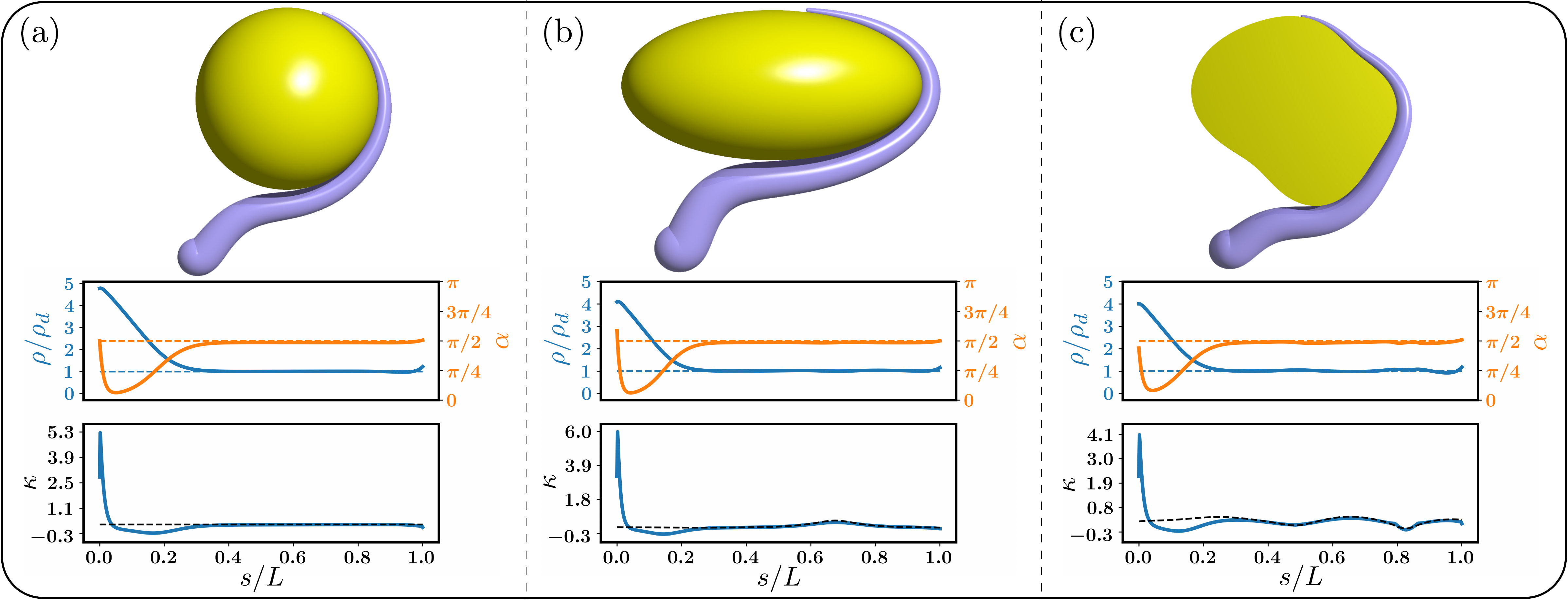}
	\caption{Optimal grasping configurations obtained by solving the optimal control problem~\eqref{eq:optimal_kappa}.
    The problem is solved for three different smooth boundary objects: (a) circle, (b) ellipse and (c) deformed circle. A rendering of the arm grasping each object is shown at the top of each column. Below this, $\rho(s)/\rho_d$ (solid blue line) and $\alpha(s)$ (solid orange line) are plotted against the normalized arc length $s/L$, with their expected convergence values indicated by blue and orange dashed lines, respectively. The arm curvature $\kappa(s)$ (solid blue line) and object curvature $\kappa_{\object}(s)$ (dashed black line) are shown in the bottom plots.}
	\label{fig:optimal}
    \vspace*{-10pt}
\end{figure*}

In this section, we solve the control problem in~Sec.~\ref{sec:control_problem} by posing it as an optimization problem. For a given desired profile $(\rho_d, \alpha_d)$, consider the following problem: 
\begin{align}
\begin{split}
\underset{\kappa(\cdot)}{\text{minimize}} ~~ &J(\kappa; (\rho_d, \alpha_d)) = \frac{1}{2}\int_0^L \kappa^2(s) \,\dif s \\
   &\hspace*{-50pt} + \int_0^L \chi\left( \frac{1}{2}(\rho(s) - \rho_d(s))^2 + \left(1 - \cos (\alpha(s) - \alpha_d (s) \right) \right)  \,\dif s \\ 
\text{subject to} \quad & \text{kinematics~\eqref{eq:rho_alpha_2}} \\
\text{and} \qquad		& (\rho, \alpha) (0) = (\rho_0, \alpha_0)~\text{given}, (\rho, \alpha)(L)~ \text{free} \\
\text{and} \qquad		& \rho(s) > 0, ~ 0 < \alpha(s) < \pi, ~~ s \in [0, L]
\end{split}
\label{eq:optimal_kappa}
\end{align}
The first term in the cost function in~\eqref{eq:optimal_kappa} serves as a regularization term and the second term is a running cost minimizing the distance between $(\rho, \alpha)$ and $(\rho_d, \alpha_d)$, and $\chi > 0$ is a weighting parameter. The optimization is subject to the kinematics~\eqref{eq:rho_alpha_2} and the constraints on the states for a feasible grasp. 

The optimization problem~\eqref{eq:optimal_kappa} is a standard optimal control problem whose solution is readily found by applying Pontryagin's Maximum Principle (PMP)~\cite{liberzon2011calculus}. Let the costate variables be denoted by $(p_1, p_2)$. Then the control Hamiltonian\footnote{The state constraints are ignored here for simplicity. These constraints are taken into account by augmenting the Lagrangian with constraint violation terms, e.g., see~\cite{chang2021controlling} for details.} is written as
\begin{align}
\begin{split}
  H &= -p_1 \cos \alpha + p_2 \left( -\kappa + \frac{\kappa_\object}{1 + \kappa_\object \rho} \sin \alpha \right) - \mathcal{L} (\rho, \alpha)  
\end{split}
\end{align}
where the Lagrangian $\mathcal{L} (\rho, \alpha)$ is the integrand in the cost function in \eqref{eq:optimal_kappa}. 
The optimal curvature is found by maximizing the control Hamiltonian $H$, i.e. setting $\frac{\partial H}{\partial \kappa} = 0$ yields
\begin{align}
\kappa = - p_2
\end{align}
Derivatives of $(p_1, p_2)$ are obtained by Hamilton's equations as
\begin{align}
\begin{split}
\frac{\dif p_1}{\dif s} &= - \frac{\partial H}{\partial \rho} 
						= \frac{\kappa_\object^2 p_2}{(1+\kappa_\object \rho)^2} \sin \alpha + \chi (\rho - \rho_d) \\
\frac{\dif p_2}{\dif s} &= - \frac{\partial H}{\partial \alpha} 
						= - p_1 \sin \alpha - \frac{\kappa_\object p_2}{1 + \kappa_\object \rho} \cos \alpha + \chi \sin (\alpha - \alpha_d)
\end{split}
\label{eq:costates}
\end{align}
We also have the transversality conditions $(p_1, p_2)(L) = (0, 0)$ due to the free endpoint condition.

\begin{remark}
In our previous work~\cite{chang2020energy, chang2021controlling}, the energy shaping control method was used to simplify the dynamic control of a soft arm. In particular, an optimal control problem in general arm coordinates (i.e. $\mathbf{r}$ and $\theta$ are state variables)  was formulated to solve for desired arm kinematic variables, including curvature (see e.g.,~\cite[equation (11)]{chang2020energy}). Even though this method was used to demonstrate grasping of a planar object with circular boundary, the effects of non-constant object boundary were not explicitly considered. The current work explicitly models the continuum grasping geometry and accounts for arbitrary object curvatures. This also allows us to pose the optimal control problem~\eqref{eq:optimal_kappa} in the reduced shape variables ($\rho$ and $\alpha$). 
\end{remark}

\subsection{Numerical results} \label{sec:numerics}
To demonstrate the proposed theoretical framework, numerical simulations are conducted to solve the kinematics of the grasping problem for three different planar boundary curves: a circle with a radius of $5$ units, an ellipse with semi-major and semi-minor axes of $8$ and $4$ units, respectively, and a deformed circle, obtained by continuously deforming the circle of radius of 5 units (see Fig.~\ref{fig:optimal}(c)). 
The arm is taken to be of length $L = 2L_\object/3$ with a (linearly) tapered radius profile $r(\cdot)$. The profile starts  with a base radius of $L/20$ and ends with a tip radius of $L/200$.
The objective is to determine the optimal configuration of the continuum arm which tracks the boundary of an object at a desired distance $\rho_d (s) = r(s)$ and $\alpha_d (s) = \tfrac{\pi}{2}$ for all $s$.

This is done by solving the state \eqref{eq:rho_alpha_2} and costate \eqref{eq:costates} equations from PMP through an iterative forward-backward algorithm~\cite{chang2020energy}. The system's kinematics \eqref{eq:rho_alpha_2} are integrated forward using the mapped curvature of the object's boundary. Subsequently, the costate equations are integrated backward to evaluate the gradient of the cost functional. The control input $\kappa$ is iteratively updated via gradient ascent as follows
\begin{align} \label{eq:forward_backward}
\begin{split}
    \kappa^{(k+1)} &= \kappa^{(k)} + \eta \frac{\partial H}{\partial \kappa}^{(k)} 
                = \kappa^{(k)} + \eta (-p_2^{(k)} - \kappa^{(k)}) 
\end{split}
\end{align}
where the superscripts $^{(k)}$ indicate the values of the variables at $k$-th iteration, and $\eta$ is a small step-size parameter. 
For each of the objects, the following parameter values are used: $\eta = 10^{-6}$, $\chi = 10$. 
The arm is initialized such that ($\rho_0,\alpha_0)$ is (5, 1.6) for the circle, (5.3, 1.8) for the ellipse, and (4.7, 1.4) for the deformed circle, respectively.

 The simulation results are shown in Fig.~\ref{fig:optimal}. For all three cases, the optimal arm successfully wraps around the boundary of the object. As the arm approaches and wraps around the object, the normalized distance $\rho/\rho_d$ and $\alpha$ become close to $1$ and $\tfrac{\pi}{2}$, respectively, indicating that the arm has achieved the given objectives.
The arm curvature $\kappa$ closely follows the object curvature $\kappa_{\object}$ once contact is established, but does not match exactly because the arm wraps outside the object boundary at a distance $\rho_d$.

\section{Algebraic grasp quality} \label{sec:quality}

Once a grasping configuration is designed, a question about the goodness of such a grasp arises naturally.
Grasp quality is a well-established concept in the grasping literature, where quantitative metrics are used to evaluate the effectiveness of a given grasp configuration~\cite{rubert2018characterisation, li1988task}. Different approaches of computing such quality metrics include magnitude of contact forces, algebraic properties of a `grasp map', and the geometric distribution of contact points (see~\cite{rubert2018characterisation, roa2015grasp} for an overview). In this section, a class of continuum grasp quality metrics is introduced, based on the algebraic properties of a continuum grasp map.


\subsection{Preliminaries: The grasp map}
We first define the `grasp map' for point contacts (see also \cite{halder2025statics, murray1994mathematical} for further details). Suppose there are $n$ contact points on the boundary at locations $\bm{\gamma}_i$, with the tangent angles $\phi_i$. Then, the mapping to transform a point contact force in the local contact frame to a wrench (force and couple) in a global frame is called the \textit{(point) grasp map} for $i$-th contact $G_i \,:\, \R^2 \rightarrow \R^3$, and is expressed as $G_i = \begin{bmatrix} I_2 \\ (\bm{\gamma}_i^\perp)^\transpose  \end{bmatrix} \mathsf{R} (\phi_i)$, where $I_2$ is the $2 \times 2$ identity matrix, $\mathsf{R}(\cdot)$ is the rotation matrix as introduced in~\eqref{eq:rho_alpha_definition}, and ${\bm{\gamma}_i}^\perp = \mathsf{R}(\pi/2) \bm{\gamma}_i$.

\begin{definition} \label{def:point_grasp_map}
    {\bf(Discrete grasp map)}  The discrete grasp map is defined by collecting all individual point grasp maps, i.e. $\tilde{G} \,:\, \R^{2n} \rightarrow \R^3$,   $\tilde{G} := [G_1 ~~ G_2 ~~ \cdots G_n]$.
\end{definition}

To extend the idea of grasp maps to the continuum setting, we make use of the distributed nature of the contact along the boundary. Recall from Definition~\ref{def:contact} that at an arm arclength $s \in [0, L]$, the variable $\delta(s)$ determines if the arm is in contact with the object or not, i.e. denote the contact set as $\mathcal{S} := \{ s \in [0, L] \,:\, \delta(s) \geq 0 \} \subset [0, L]$. Further, denote the local contact forces at a contact point $s$ by $f(s) \in \R^2$, and assume also that the function $f(\cdot)$ is square integrable, i.e. $f \in \mathrm{L}_2(\mathcal{S}; \R^2) =: \mathcal{U}$. 
\begin{definition}
    {\bf(Continuum grasp map)} The mapping to transform all (local) contact forces to a wrench in the global frame is called the continuum grasp map $\mathcal{G}: \mathcal{U} \rightarrow \R^3$, and is expressed as
    \begin{align}
        \mathcal{G} \cdot f = \int_\mathcal{S} G(s) f(s) \, \dif s
        \label{eq:grasp_map}
    \end{align}
    where $G(s)$ is the point grasp map at point $s$. 
\end{definition}


\subsection{Continuum grasp quality}
Three different grasp quality metrics are defined in the discrete case, based on the algebraic properties of the grasp map $\tilde{G}$ (see a summary in~\cite{rubert2018characterisation, roa2015grasp}):

\begin{itemize}
    \item Smallest singular value of $\tilde{G}$, $\sigma_{\min} (\tilde{G})$~\cite{li1988task}: This metric measures how far a grasp is from a singular configuration, i.e., the grasp loses capability of resisting wrenches in at least one direction.
    \item Volume of $\tilde{G}$ in the wrench space, $\Pi_i \, \sigma_i(\tilde{G})$~\cite{li1988task}: This metric measures the volume of the ellipsoid generated by all reachable wrenches from all contact forces of unit magnitude.
    \item Grasp Isotropy Index, $\sigma_{\min} (\tilde{G})/ \sigma_{\max} (\tilde{G})$~\cite{kim2001optimal}: This metric emphasizes uniform contribution from each contact point.
\end{itemize}

To extend these metrics in the continuum setting, the singular values of the map $\mathcal{G}$ need to be computed. We show below that these singular values are directly related to the eigenvalues of a Gramian matrix. 

\begin{proposition} \label{prop:singular}
    Define the Gramian matrix 
    \begin{align}
        W := \int_\mathcal{S} G(s) G^\transpose(s) \, \dif s
    \end{align}
    Then, the singular values of $\mathcal{G}$ are related to the eigenvalues of $W$ as $\sigma(\mathcal{G}) = \sqrt{\lambda(W)}$.
\end{proposition}

\smallskip
\textit{Proof.}~~
See Appendix~\ref{appdx:singular_proof}.

\smallskip
\begin{remark} {\bf A linear systems approach.} As shown in~\cite[Sec.~III]{halder2025statics}, the Gramian matrix $W$ corresponds to a linear drift-less control system $\frac{\dif x}{\dif s} = Gf$, where the state $x$ is the cumulative global wrench and the contact forces $f$ are considered as the control inputs. 
\end{remark}

Based on Prop.~\ref{prop:singular}, we can then introduce the continuum grasp quality metrics as follows:

\begin{itemize}
    \item Smallest eigenvalue of $W$, $Q_1 := \lambda_{\min} (W) \geq 0$.
    \item Determinant of $W$, $Q_2 := \det (W) =  \Pi_i \, \lambda_i(W) \geq 0$.
    \item Inverse condition number of $W$,~$Q_3:= {\lambda_{\min} (W)}/{ \lambda_{\max} (W)} \in [0, 1]$.
\end{itemize}
Each of these metrics retain their physical significances as their discrete counterparts.

\subsection{Maximizing grasp quality}
The continuum grasp qualities $Q_i$ depend on the parts of the object under the continuum grasp, i.e. $\bm{\gamma}(s), \, s \in \mathcal{S}$, which in turn depend on the direction the arm approaches the object from. 
Therefore, these quality metrics can be written as functions of the arm base configuration, i.e., $Q_i = Q_i (\mathbf{r}_0, \theta_0)$, where $\mathbf{r}_0 \in \R^2$ is the base location and $\theta_0 \in S^1$ is the base orientation. 

To systematically assess the dependence of all three quality metrics on the arm base configurations, a `disc'-shaped region $\mathcal{D}$ around the object is created as $\mathcal{D} = \{ \mathbf{r}_0 \in \R^2 \,:\, \text{dist}(\mathbf{r}_0, \gamma(\cdot)) = d, ~ d \in [d_{\min}, d_{\max} ] \}$, where $\text{dist}(\cdot, \cdot)$ is the distance from the disc to the object boundary $\gamma(\cdot)$. The arm base is placed at $\mathbf{r}_0 \in \mathcal{D}$, represented in the polar coordinates as $(\norm{\mathbf{r}_0}, \psi)$. The base orientation is chosen so as to make the arm base `tangent' to the disc, i.e. $\theta_0 = \psi + \tfrac{\pi}{2}$. 
Then, an optimization problem is posed to find the optimal arm approach and grasp configuration for achieving the highest grasp quality
\begin{equation}
    \underset{\mathbf{r}_0 \in \mathcal{D}}{\text{maximize}}~~ Q_i(\mathbf{r_0}), \;\;\;i = 1,2,3
    \label{eq:max_quality}
\end{equation}


To numerically demonstrate the results, discs of minimum distance $d_{\min} = 3$ and maximum distance $d_{\max} = 13$ are created for each of the three objects described in Sec.~\ref{sec:numerics}. The arm is chosen to be of a constant unit radius $r(s) = 1$ and a length of $L_\object/2$, yielding $(\rho_d, \alpha_d) = (1, \tfrac{\pi}{2})$ for the inner-loop problem.
The inner-loop minimization problem is solved using the same forward-backward algorithm as in Sec.~\ref{sec:numerics}, while the MATLAB function \textit{fminsearch} is used to solve the outer-loop problem to minimize $-Q_i$, $i = 1,2,3$.


 

\begin{figure*}[t] 
	\centering
	\includegraphics[width=0.8\textwidth]{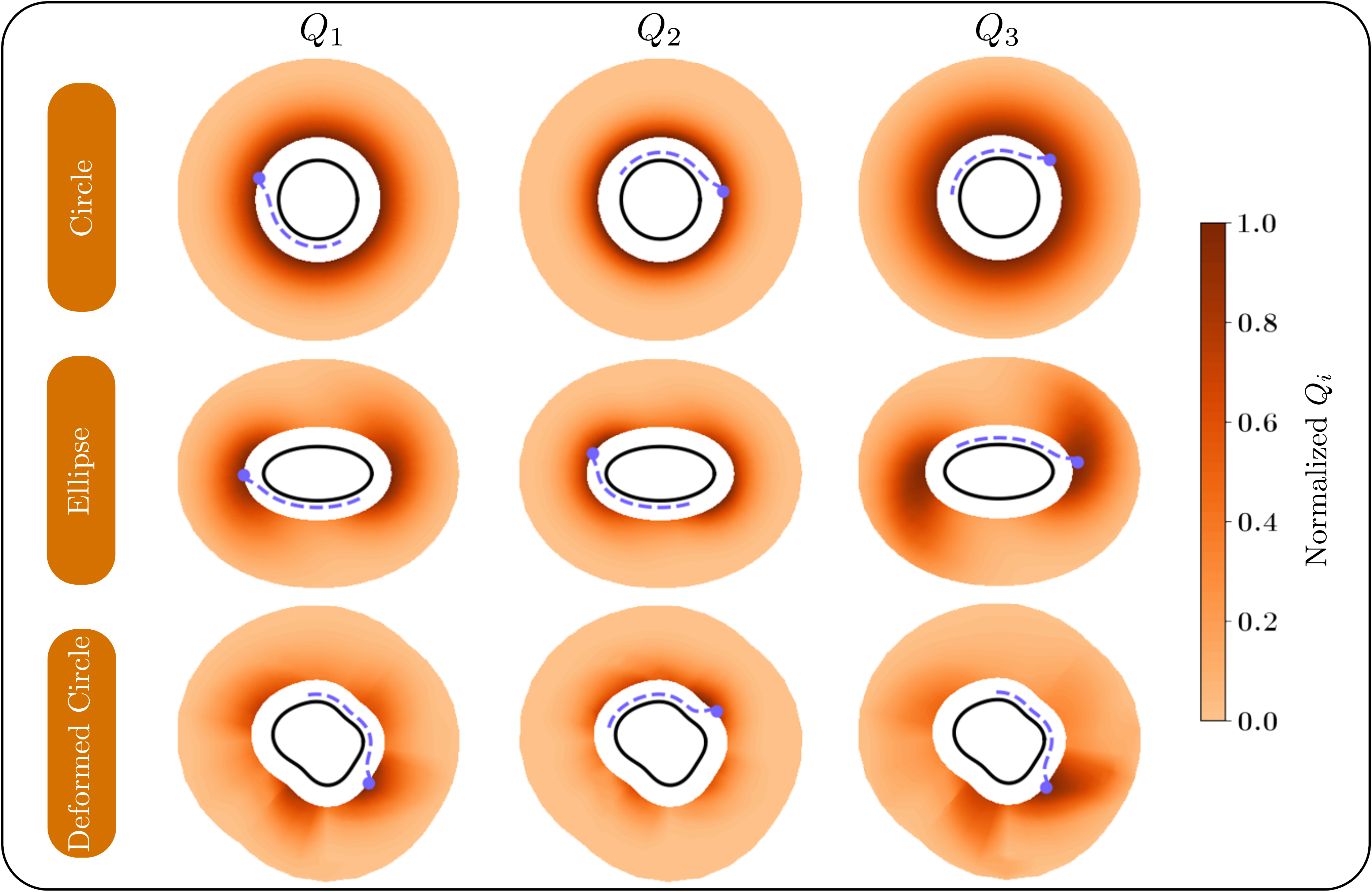}
	\caption{Grasp quality maps for the metrics $Q_1$, $Q_2$ and $Q_3$ for the circular, elliptical, and deformed circular objects. The quality values are normalized by the maximum quality for each object, where darker regions denote higher quality. The purple point indicates the optimal starting location, and the corresponding arm configuration is shown as a purple dashed curve.}
    \label{fig:max_grasp}
    \vspace*{-12pt}
\end{figure*}


The numerical results are shown in Fig.~\ref{fig:max_grasp}. 
The quality metric heatmaps (normalized for each metric) for the region $\mathcal{D}$ are shown for each object and the optimal grasps found by solving~\eqref{eq:max_quality} are drawn in  purple dashed lines. As expected, grasp quality increases when the arm starts closer to the object, allowing more of the object to be grasped. 
A radial symmetry in each quality metric for the circular boundary is also observed. 
The grasp quality is higher near the bulged areas of the ellipse, with symmetric (about the major and minor axes) patterns emerging for all three metrics (notice slight deviations for $Q_3$)
For the deformed circle, the grasp quality is higher near the concave regions of the object.

\section{Toward a dynamic feedback control} \label{sec:dynamics}

While the optimal control formulation provides a principled solution to the control problem in Sec.~\ref{sec:control_problem}, its computational complexity may limit real-time implementation, especially in a dynamic setting. In this section, we explore a plausible pathway to a dynamic feedback control for a continuum arm to grasp an object. 

\subsection{Dynamic arm model}
We begin by providing a brief dynamic model of the soft arm. To do so, we make use of a temporal variable, denoted by $t$, in addition to the spatial variable $s$ -- already defined in Sec.~\ref{sec:model}. Also, for brevity, we use the notations $(\cdot)_t$ and $(\cdot)_s$ for partial derivatives with respect to $t$ and $s$, respectively. 

The dynamics of a soft arm are described by a set of partial differential equations, following the Cosserat theory of elastic rods as follows~\cite{antman1995nonlinear, gazzola2018forward}:
\begin{equation}
	\begin{aligned}
		(\varrho A \mathbf{r}_t)_t &= \mathbf{n}_s - \zeta \mathbf{r}_t + \mathbf{f}^{\text{drag}} \\
		(\varrho I \theta_t)_t &= (EI\kappa)_s + n_2  - \zeta \theta_t + u_s 
	\end{aligned}
	\label{eq:dynamics}
\end{equation}
Here, specific variables of interest are: (i) internal forces $\mathbf{n} = n_1 \mathbf{a} + n_2 \mathbf{b}$, that are determined by the rod's inextensibility and unshearability constraints;
(ii) internal couple due to (linear) elasticity $EI \kappa$, where $E$ and $I$ are the Young's modulus and second moment of area of the cross section of the arm, respectively. The variable $u$ is regarded as internal \textit{couple control} (making the total internal couple $m = EI\kappa + u$), which may be generated by muscle actuations~\cite{chang2021controlling, wang2022sensory}. Details of the remaining variables are given in Appendix~\ref{appdx:dynamics}. 
Finally, the dynamics \eqref{eq:dynamics} are accompanied by a fixed-free boundary condition 
\begin{align}
\mathbf{r}(0, t) = \mathbf{0},~ \theta(0, t) = 0, ~\mathbf{n}(L, t) = \mathbf{0},~ m(L, t) = 0
\label{eq:boundary_conditions}
\end{align}

\subsection{Feedback control and its analysis}
With the specified dynamic model, we propose the following feedback control law 
\begin{align}
    u = - {EI}(- \mu_1 \cos \alpha + (\rho - \mu_2) \sin \alpha )
    \label{eq:dynamic_feedback}
\end{align}
where the variables $\rho$ and $\alpha$ are considered as feedback (see definition in~\eqref{eq:rho_alpha_definition}), and $\mu_1, \mu_2$ are control gains. 
To analyze the effect of this feedback control, we first consider the statics of the arm under the proposed feedback (i.e., the closed-loop kinematics), followed by a dynamic convergence analysis. 

The statics of the arm are readily found by setting all time-derivatives in \eqref{eq:dynamics} to zero and using the boundary conditions~\eqref{eq:boundary_conditions} as
\begin{align}
    \kappa = -\frac{u}{EI} = - \mu_1 \cos \alpha + (\rho - \mu_2) \sin \alpha 
    \label{eq:curvature_feedback}
\end{align}
The kinematics at the closed-loop equilibrium are then given by the non-autonomous system~\eqref{eq:rho_alpha_2}, along with curvature~\eqref{eq:curvature_feedback}. We immediately have the following result. 

\begin{proposition} \label{prop:feedback}
Suppose the object curvature is a constant, $\kappa_\object (s) = \bar{\kappa} \geq 0, \, s \in [0, L]$, i.e., the segment under contact is a circular arc or a straight line. Then, we may choose $\mu_1 >0, \mu_2 > - \bar{\kappa}$, such that $(\rho, \alpha)$ asymptotically converge to $\left(\rho_d, \tfrac{\pi}{2}\right)$, where $\rho_d > 0$ is a desired contact distance. Moreover, the desired contact distance $\rho_d$ is a root of the quadratic equation
\begin{align}
    \bar{\kappa} \rho^2 + (1-\mu_2 \bar{\kappa}) \rho - (\mu_2 + \bar{\kappa}) = 0
    \label{eq:quadratic}
\end{align}
\end{proposition}

\smallskip
\textit{Proof.}~~
See Appendix~\ref{appdx:feedback_proof}.

Proposition~\ref{prop:feedback} exhibits the efficacy of the feedback control for a special case of constant curvature object segment. The control law~\eqref{eq:dynamic_feedback} is inspired from a boundary following steering control for a moving vehicle~\cite{zhang2004boundary}. The intuition is that the first term in the control law~\eqref{eq:dynamic_feedback} makes the arm tangent $\mathbf{a}$ vector parallel to the object tangent $\Tangent (s)$ (i.e., $\alpha \rightarrow \tfrac{\pi}{2}$) and the second term serves to make the contact distance converge to a desired value (i.e., $\rho \rightarrow \rho_d$).

Finally, we state a dynamic convergence result as follows.
\begin{proposition} \label{prop:dyanmic_feedback}
    Consider the dynamics of the arm \eqref{eq:dynamics}-\eqref{eq:boundary_conditions}, along with the feedback control~\eqref{eq:dynamic_feedback}. Then, the equilibrium given by \eqref{eq:rho_alpha_2}, \eqref{eq:curvature_feedback} is (locally) asymptotically stable.
\end{proposition}

\smallskip
\textit{Proof.}~~
See Appendix~\ref{appdx:dynamic_feedback_proof}.

\begin{figure}[t]
    \hspace*{1pt}
	\includegraphics[width=0.48\textwidth, trim = {0pt 0 0pt 0}, clip = false]{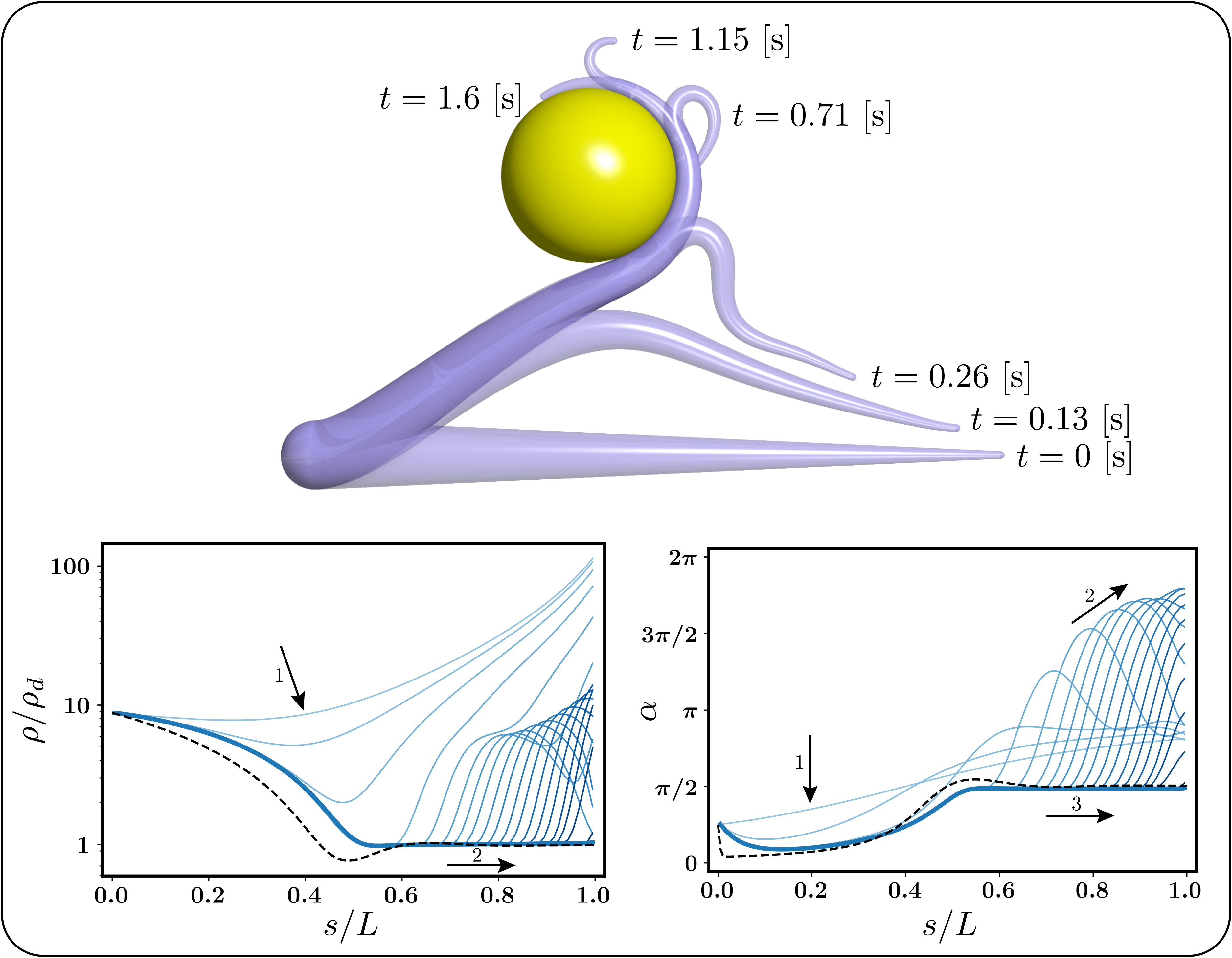}
	\caption{Dynamic simulation results for a soft arm grasping a fixed circular object. The top section shows six time instances of the arm in transparent purple. The bottom section shows a series of time snapshots of $\rho/\rho_d$ and $\alpha$ in blue. The black dashed lines represent the theoretical static solutions, corresponding to the curvature~\eqref{eq:curvature_feedback}.
    Black arrows illustrate the time progression of $\rho/\rho_d$ and $\alpha$. 
    }
	\label{fig:dynamics}
    \vspace*{-10pt}
\end{figure}

\subsection{Dynamic simulation results} \label{sec:dynamic_results}
This subsection presents the results of a dynamic simulation of a soft arm grasping a static circular object. 
The dynamics of the rod \eqref{eq:dynamics}-\eqref{eq:boundary_conditions} are numerically solved using the open-source software \textit{PyElastica}~\cite{gazzola2018forward, PyElastica}. The arm is modeled as a tapered rod with a length of 20 [cm], a base radius of 1 [cm] and a tip radius of 0.1 [cm]. The arm is tasked with wrapping around a circular object of radius 2.5 [cm], with the desired separation distance $\rho_d$ profile being the same as the tapered radius profile of the arm. The arm accomplishes the grasping task by applying the dynamic feedback control law~\eqref{eq:dynamic_feedback}. Additional simulation details are provided in Appendix~\ref{appdx:dynamics}.


As can be seen in Fig. \ref{fig:dynamics}, the rod initially bends towards the closest point on the boundary of the object, then the tip curls and progressively wraps around the object until full contact is achieved. During wrapping, ${\rho/\rho_d}$ converges to 1 and $\alpha$ becomes close to, but not exactly $\tfrac{\pi}{2}$ due to the tapering of the arm. 
Minor discrepancies against the theoretical static configuration are observed, due to arm not being able to pass through the object (modeled as an obstacle here) and the tapering of the arm. 


 
\vspace*{-5pt}
\section{Conclusion} \label{sec:conclusion}
This paper presents a control-theoretic framework for analyzing the kinematics of planar grasping by a soft continuum arm. By viewing the grasping problem as a boundary-following problem, the kinematics are expressed in terms of geometric shape variables, including the distance and bearing to the boundary. The reduced kinematics are considered as a non-autonomous control system, with the arm curvature serving as the control input. An optimal control problem is then posed to construct optimal grasping configuration for a given object, whose solution is also characterized. The resulting continuum grasp map leads to the introduction of a class of grasp quality metrics, which are generalizations of quality metrics in the rigid-body grasping case. Furthermore, a pathway toward utilizing the kinematic analysis in synthesizing dynamic feedback controls is provided. All of the theoretical developments are accompanied by illustrative numerical examples. 

Future work will incorporate contact mechanics to model the contact forces along the contact interface. Integrating the current geometric framework with dynamic Cosserat rod models~\cite{chang2020energy, chang2023energy} will be of interest. 
Another important direction is further development of the feedback control strategy proposed here to utilize feedback information from tactile sensing and proprioception~\cite{wang2025neuralbiocyber}. 

\bibliographystyle{IEEEtran}
\bibliography{bibfiles/halder_papers,bibfiles/reference}

\appendices
\renewcommand{\thelemma}{A-\arabic{section}.\arabic{lemma}}
\renewcommand{\thetheorem}{A-\arabic{section}.\arabic{theorem}}
\renewcommand{\theequation}{A-\arabic{equation}}
\renewcommand{\thedefinition}{A-\arabic{definition}}
\setcounter{lemma}{0}
\setcounter{theorem}{0}
\setcounter{equation}{0}

\vspace*{-5pt}
\section{Proof of Proposition \ref{prop:singular}} \label{appdx:singular_proof}
\begin{proof}
Indeed, by definition, the singular values of $\mathcal{G}$ are the positive square roots of the self-adjoint operator $\mathcal{G}^* \mathcal{G} \,:\, \mathcal{U} \rightarrow \mathcal{U}$, where $\mathcal{G}^* \,:\, \R^3 \rightarrow \mathcal{U}$ is the adjoint of $\mathcal{G}$. It is obvious that the adjoint map $\mathcal{G}^*$ is expressed as $\mathcal{G}^* w = G^\transpose(s) w$, where $w \in \R^3$. Then, we find the eigenvalues $\lambda$ of $\mathcal{G}^*\mathcal{G}$ from the definition $\mathcal{G}^*\mathcal{G} f = \lambda f, \, f \in \mathcal{U}$, as
\begin{align}
    G^\transpose (s) \int_\mathcal{S} G(\bar{s}) f(\bar{s})\, \dif \bar{s} = \lambda f(s), \quad s \in \mathcal{S}
    \label{eq:eigenvalue}
\end{align}
Denote $w := \int_\mathcal{S} G(\bar{s}) f(\bar{s})\, \dif \bar{s} \in \R^3$. Then, multiplying both sides of \eqref{eq:eigenvalue} by $G(s)$ and integrating in the set $\mathcal{S}$, we obtain
\begin{align*}
    \left(\int_\mathcal{S} G(s) G^\transpose (s) \,\dif s \right) w = \lambda \int_\mathcal{S} G({s}) f({s})\, \dif s = \lambda w
\end{align*}
that is $W w = \lambda w$. Therefore, $\lambda$ is an eigenvalue of the Gramian matrix $W$.
\end{proof}

\vspace*{-5pt}
\section{Proof of Proposition \ref{prop:feedback}} \label{appdx:feedback_proof}
\begin{proof}
Consider a Lyapunov function candidate 
\begin{align*}
    V(\rho, \alpha) = - \ln (\sin \alpha) + \frac{1}{2}(\rho - \mu_2)^2 - \ln (1+\rho \bar{\kappa})
\end{align*}
Notice that $V$ is well-defined in the admissible region of $(\rho, \alpha)$. Taking derivative of $V$ along the trajectories of \eqref{eq:rho_alpha_2}, \eqref{eq:curvature_feedback} yields 
\begin{align*}
    \frac{\dif V}{\dif s} &= \frac{\partial V}{\partial \rho} \frac{\dif \rho}{\dif s} +  \frac{\partial V}{\partial \alpha} \frac{\dif \alpha}{\dif s} = - \mu_1 \frac{\cos^2 \alpha}{\sin \alpha} \leq 0
\end{align*}
Further, we see that the largest invariant set such that $\tfrac{\dif V}{\dif s} = 0$ is characterized by 
\begin{align} \label{eq:dvds}
    -(\rho - \mu_2) + \frac{\bar{\kappa}}{1 + \rho \bar{\kappa}} = 0
\end{align}
which simplifies to \eqref{eq:quadratic}. Thus, the statement of the Proposition follows by the application of LaSalle's theorem~\cite{khalil2002nonlinear}.
\end{proof}

\vspace*{-10pt}
\section{Proof of Proposition \ref{prop:dyanmic_feedback}} \label{appdx:dynamic_feedback_proof}
\begin{proof}
We first show that the internal control term is a gradient of an energy function. Note that for any time $t$, given a curvature profile $\kappa$, and object curvature $\kappa_\object$, $z=(\rho, \alpha)$ can be uniquely determined, i.e. we may express $z = z(\kappa; \kappa_\object)$. Define $\mathcal{W}= \int [- \mu_1 \cos \alpha (\kappa) + (\rho (\kappa) - \mu_2) \sin \alpha (\kappa)] \, \dif \kappa$. Then the control $u$ is clearly the gradient of $\mathcal{W}$. The proof then follows by applying an energy shaping argument (see~\cite{wang2022sensory, chang2021controlling} for details) that utilizes the damped Hamiltonian structure of the rod dynamics~\eqref{eq:dynamics}. 
\end{proof}

\vspace*{-0pt}
\section{Details of the dynamic model and simulation} \label{appdx:dynamics}
In the arm dynamics~\eqref{eq:dynamics}, $\varrho$ denotes the density of the arm, $A$ is the area of the cross-section, $\zeta$ is a damping co-efficient, and $\mathbf{f}^{\text{drag}}$ is a force due to fluidic drag.

The center of the circular object is positioned 8 [cm] above and 8 [cm] to the right of the base of the rod. The object is modeled as an obstacle, providing rudimentary repulsive contact forces~\cite{gazzola2018forward}. To execute the feedback control~\eqref{eq:dynamic_feedback}, $\mu_1$ is chosen to be 1 and $\mu_2$ is chosen in an adaptive manner to account for the varying radius $r(s)$ of the tapered arm 
\begin{align*}
\mu_2(s)=r(s)-\frac{1}{r_{\text{obj}}+r(s)}
\label{eq:dynamic_mu}
\end{align*}
The expression for $\mu_2(s)$ is obtained from \eqref{eq:dvds} by setting $\rho = \rho_d = r_{\text{arm}}(s)$ and $\bar{\kappa} = 1/r_{\text{obj}}$.


The rest of the parameter values used for the arm and drag model are obtained from~\cite{wang2022sensory}, based on real octopus measurements. 





\end{document}